\documentclass[twoside]{article}

\usepackage[accepted]{aistats2026}
\usepackage[T1]{fontenc}
\usepackage[protrusion=false]{microtype}
\usepackage{graphicx}
\usepackage{booktabs}
\usepackage{multicol}
\usepackage{multirow}
\usepackage[table, dvipsnames]{xcolor}
\usepackage{hyperref}
\usepackage{amsmath}
\usepackage{amssymb}
\usepackage{amsthm}
\usepackage{tcolorbox}
\usepackage[round]{natbib}
\usepackage{bbding}
\usepackage[weather, misc]{ifsym}
\usepackage{svg}

\usepackage{tikz}
\usetikzlibrary{positioning, shapes, arrows.meta, calc}

\newcommand{\vega}{\textsc{GiVA}}
\newcommand{\papertitle}{Gradient-Informed Bases for Vector-Based Adaptation}
\newcommand{\initvrur}[1][\tiny]{{#1\textsc{TopSV-B}}}
\newcommand{\initvrutwor}[1][\tiny]{{#1\textsc{SecSV-B}}}
\newcommand{\initvr}[1][\tiny]{{#1\textsc{Rand-B}}}
\newcommand{\olmo}{OLMo~2}
\newcommand{\mistral}{Mistral}
\newcommand{\qwen}{Qwen~2}
\newcommand{\msphi}{Phi~3}
\newcommand{\dinob}{DINOv2-ViT-B/14}
\newcommand{\clipl}{CLIP-ViT-L/14}
\newcommand{\vegacolor}{gray!20}
\newcommand{\pt}{\textrm{pt}}
\newcommand{\ovector}{\Gamma}
\newcommand{\svector}{\Lambda}

\DeclareMathOperator*{\argmin}{arg\,min}

\newtheorem{theorem}{Theorem}

\begin{document}

%

%
\runningauthor{Gangwar, Deshmukh, Shavlovsky, Li, Mittal, Ying, Kani}

\twocolumn[

\aistatstitle{\vega{}: \papertitle{}}

\aistatsauthor{Neeraj Gangwar\textsuperscript{\normalfont \textdagger}
\quad
Rishabh Deshmukh\textsuperscript{\normalfont \S}
\quad
Michael Shavlovsky\textsuperscript{\normalfont \S}
\quad
Hancao Li\textsuperscript{\normalfont \S}}
\aistatsauthor{Vivek Mittal\textsuperscript{\normalfont \S}
\quad
Lexing Ying\textsuperscript{\normalfont $\mathparagraph$}
\quad
Nickvash Kani\textsuperscript{\normalfont  \textdagger}\vspace{0.5em}}

\aistatsaddress{\textsuperscript{\textdagger}University of Illinois Urbana-Champaign
\quad
\textsuperscript{\S}Amazon
\quad
\textsuperscript{$\mathparagraph$}Stanford University
\\
{\small \texttt{gangwar2@illinois.edu}, \texttt{derishab@amazon.com}}}
]

\begin{abstract}
\hyphenpenalty=2000
As model sizes continue to grow, parameter-efficient fine-tuning has emerged as a powerful alternative to full fine-tuning. While LoRA is widely adopted among these methods, recent research has explored vector-based adaptation methods due to their extreme parameter efficiency. However, these methods typically require substantially higher ranks than LoRA to match its performance, leading to increased training costs. This work introduces \vega{}, a gradient-based initialization strategy for vector-based adaptation. It achieves training times comparable to LoRA and maintains the extreme parameter efficiency of vector-based adaptation. We evaluate \vega{} across diverse benchmarks, including natural language understanding, natural language generation, and image classification. Experiments show that our approach consistently outperforms or achieves performance competitive with existing vector-based adaptation methods and LoRA while reducing rank requirements by a factor of eight ($8\times$).
\end{abstract}
\section{INTRODUCTION}\label{sec:introduction}
Large language models (LLMs) have demonstrated remarkable success through large-scale pre-training, enabling them to learn rich representations of language and general world knowledge \citep[e.g.,][]{olmo20242, liu2024deepseek, team2024gemma, achiam2023gpt, team2024qwen2, Abdin2024Phi3TR}. However, to fully leverage their capabilities in real-world applications, these models typically need to be fine-tuned to adapt them effectively to various complex downstream tasks, such as mathematical and commonsense reasoning, code understanding and generation, etc. Fully fine-tuning these models for each application is extremely costly, especially as model sizes continue to grow, making it computationally expensive. Moreover, each fine-tuned copy requires the same amount of storage as the original model, which is further amplified by the need to store multiple checkpoints.

To address these challenges, parameter-efficient fine-tuning (PEFT) has emerged as a powerful alternative, offering a better trade-off between the computational cost and performance \citep[e.g.,][]{Hu2021LoRALA, li2021prefix, lester2021power, houlsby2019parameter, sung2021training}. Recent extensions have further advanced this line of work \citep[e.g.,][]{zhang2023adalora, liu2024dora, kopiczko2024vera, albert2025randlora}. Among these methods, LoRA \citep{Hu2021LoRALA} has become one of the most widely adopted approaches for adapting LLMs to various downstream tasks. It leverages the observation from \citet{aghajanyan2021intrinsic} that parameter updates often lie in a low-dimensional subspace and models them as a product of two low-rank matrices, significantly reducing the number of trainable parameters.

\begin{figure*}[t]
    \centering
    \scalebox{0.8}{\begin{tikzpicture}[font={\fontfamily{phv}\selectfont\footnotesize}, text=darkgray!80, draw=darkgray!80]
    \node [draw, fill=Cerulean!60, minimum height=2.8cm, minimum width=2.8cm, inner sep=0, outer sep=0] (PRETRAINED) {\begin{tabular}{c}
            \\
            \textcolor{white}{$W_\textrm{\tiny residual} \in \mathbb{R}^{m \times d}$} \\
            \\
            \textcolor{white}{\SnowflakeChevron}
        \end{tabular}
    };

    \node [draw, fill=Cerulean!60, trapezium, trapezium stretches=true, trapezium angle=60, minimum width=2.8cm, minimum height=1cm, anchor=south west, inner sep=0, outer sep=0] (A) at ($(PRETRAINED.south east) + (1.2cm, 0cm)$) {\textcolor{white}{$A$ \raisebox{-0.2\height}\SnowflakeChevron}};

    \node [draw, fill=BurntOrange!60, minimum width=1.8cm, minimum height=0.25cm, label=right:$\svector$, anchor=south, inner sep=0, outer sep=0] (S) at ($(A.north) + (0cm, 0.1cm)$) {};

    \node [draw, fill=Cerulean!60, trapezium, trapezium stretches=true, trapezium angle=120, minimum width=2.8cm, minimum height=1cm, anchor=south, inner sep=0, outer sep=0] (B) at ($(S.north) + (0cm, 0.1cm)$) {\textcolor{white}{$B$ \raisebox{-0.2\height}\SnowflakeChevron}};

    \node [draw, fill=BurntOrange!60, minimum width=2.8cm, minimum height=0.25cm, label=right:$\ovector$, anchor=south, inner sep=0, outer sep=0] (O) at ($(B.north) + (0cm, 0.1cm)$) {};

    \node [draw, fill=Goldenrod!60, minimum width=1.74cm, minimum height=0.2cm, label=left:$x \in \mathbb{R}^{d}$] (X) at ($(PRETRAINED.south)!0.5!(A.south) + (0cm, -0.75cm)$) {};

    \node [draw, fill=Goldenrod!60, minimum width=1.74cm, minimum height=0.2cm, label=left:$h \in \mathbb{R}^{m}$] (H) at ($(PRETRAINED.north)!0.5!(O.north) + (0, 0.75cm)$) {};

    \node [below=of H, below=0.1cm] {\bf $+$};

    \draw [-{Triangle[width=8pt, length=5pt, line width=0.5pt]}, line width=5pt] ($(X.north west) + (0, 0.05cm)$) -- ($(PRETRAINED.south) - (0, 0.05cm)$);

    \draw [-{Triangle[width=8pt, length=5pt, line width=0.5pt]}, line width=5pt] ($(X.north east) + (0, 0.05cm)$) -- ($(A.south) - (0, 0.05cm)$);

    \draw [-{Triangle[width=8pt, length=5pt, line width=0.5pt]}, line width=5pt] ($(PRETRAINED.north) + (0, 0.1cm)$) -- ($(H.south west) - (0, 0.05cm)$);

    \draw [-{Triangle[width=8pt, length=5pt, line width=0.5pt]}, line width=5pt] ($(O.north) + (0, 0.1cm)$) -- ($(H.south east) - (0, 0.05cm)$);

    \node [draw, fill=Cerulean!60, minimum height=0.2cm, minimum width=0.2cm, , label={[font=\fontfamily{phv}\selectfont\scriptsize]right:Frozen}] (LABEL_FROZEN) at ($(A.south east) + (0cm, -0.5cm)$) {};

    \node [draw, fill=BurntOrange!60, minimum height=0.2cm, minimum width=0.2cm, , label={[font=\fontfamily{phv}\selectfont\scriptsize]right:Trainable}, below=of LABEL_FROZEN, below=0.1cm] {};

    \node [right=of S, right=1.5cm, yshift=1.25cm] (VERA) {
        \begin{tabular}{l}
        \textcolor{darkgray!90}{\textbf{VeRA}} \\
        $A, B \sim \operatorname{Kaiming}$, \\
        $\svector \leftarrow \mathbf{1}$, \\
        $\ovector \leftarrow \mathbf{0}$, \\
        $W_\textrm{residual} \leftarrow W_\pt$ \\
        \end{tabular}
    };

    \node [right=of VERA, right=0.0cm] (OSORA) {
        \begin{tabular}{l}
        \textcolor{darkgray!90}{\textbf{OSoRA}} \\
        $U, \Sigma, V \leftarrow \operatorname{SVD}(W_\pt)$ \\
        $B \leftarrow U_r, A \leftarrow V_r^T$ \\
        $\svector \leftarrow \Sigma_r, \ovector \leftarrow \mathbf{1}$ \\
        $W_\textrm{residual} \leftarrow W_\pt - U_r \Sigma_r V_r^T$
        \end{tabular}
    };

    \node [below=of VERA.south west, anchor=north west, yshift=0.75cm] (VEGA) {
        \begin{tabular}{l}
        \textcolor{darkgray!90}{\textbf{\vega{}} (ours)} \\
        $U, \Sigma, V \leftarrow \operatorname{SVD}\left( \nabla_W \mathcal{L}\left(W_\pt\right) \right)$ \\
        $A \leftarrow V_r^T, B^T B = \mathbb{I}^{r \times r}$ \\
        $\svector \leftarrow \mathbf{1}, \ovector \leftarrow \mathbf{0}$ \\
        $W_\textrm{residual} \leftarrow W_\pt$ \\
        \end{tabular}
    };
    
\end{tikzpicture}}
    \caption{Overview of vector-based adaptation methods. Here, $W_\pt$ denotes the pre-trained weights, $U_r$ and $V_r$ denote the first $r$ columns of $U$ and $V$, respectively, and $\Sigma_r$ corresponds to the top-left $r \times r$ submatrix of $\Sigma$.}
    \label{fig:overview}
\end{figure*}
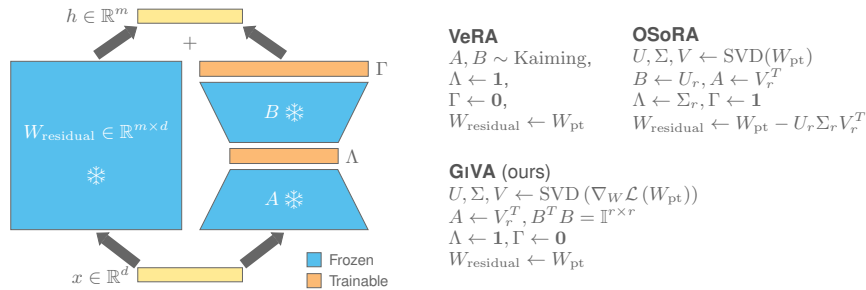

Vector-based adaptation methods (Figure~\ref{fig:overview}) extend this idea by introducing scaling vectors that adapt a pair of frozen matrices (also referred to as bases) to guide the model's behavior \citep{kopiczko2024vera, Han2025OSoRAOA}. Since only the scaling vectors are trained, they reduce the number of trainable parameters to an even greater extent than LoRA-like approaches, making them extremely parameter- and storage-efficient. This efficiency is particularly appealing in resource-constrained applications, such as scenarios where model updates must be communicated over a network (e.g., federated learning or multi-device fine-tuning), and in mixture-of-experts \citep[e.g.,][]{wu2024mixture}. However, to match LoRA's performance, these methods typically require substantially higher ranks, which leads to considerably longer training times as model sizes increase. For instance, fine-tuning \qwen{} (0.5B) on 15K commonsense reasoning examples from \citet{hu2023llm} using VeRA \citep{kopiczko2024vera} requires approximately $2.5\times$ the runtime of LoRA to achieve comparable performance (see Section~\ref{sec:training_time_comparison} for details). This additional overhead is primarily due to VeRA's higher rank—1024 versus 16 in LoRA. Motivated by this, we focus on the following question:
\begin{center}
\textit{Can we reduce the required rank in vector-based adaptation methods to achieve training times comparable to LoRA while preserving their performance?} 
\end{center}

VeRA \citep{kopiczko2024vera} uses a pair of random matrices shared across layers as its bases, whereas OSoRA \citep{Han2025OSoRAOA} derives the bases directly from pre-trained weights. As these bases are not updated during fine-tuning and contain no information about the downstream task, both methods require higher ranks than LoRA to achieve comparable performance. The choice of bases may be particularly suboptimal for OSoRA, as the reliance on pre-trained weights makes it less flexible than VeRA. We hypothesize that carefully choosing the bases may reduce the required ranks while achieving performance comparable to these methods. To this end, we propose \MakeLowercase{\papertitle{}} (\vega{}). Our approach constructs its bases using the first-step full fine-tuning gradient for the downstream task. Specifically, the bases are chosen such that the first-step gradient of \vega{} closely approximates that of full fine-tuning. This strategy is analogous to those adopted in recent methods for improving LoRA \citep{wang2024lora, wang2025lorapro}. However, unlike these approaches, we fine-tune only the scaling vectors, thereby preserving the extreme parameter efficiency of vector-based adaptation methods.

We extensively evaluate \vega{} across natural language understanding, natural language generation, and image classification tasks. Our experiments demonstrate that it consistently outperforms or achieves performance competitive with existing vector-based adaptation methods and LoRA while operating with substantially lower ranks. As a result, it attains runtimes on par with LoRA and requires an extremely small number of trainable parameters.


\paragraph{Contributions.}
The contributions of our work are as follows:
\vspace{-0.75em}
\begin{itemize}
    \item We introduce \vega{}, a vector-based adaptation method that derives the bases from the first-step full fine-tuning gradient, enabling efficient adaptation with reduced rank requirements. The reduced rank requirements result in training times comparable to those of LoRA-like methods.

    \item We evaluate our approach extensively across various benchmarks and model sizes, which show that \vega{} outperforms or achieves performance competitive with existing methods.

    \item We examine three alternative strategies for initializing the bases, each leading to similar performance across benchmarks.
\end{itemize}

\vspace{-0.75em}

Our code is publicly available on GitHub.\footnote{\url{https://github.com/neerajgangwar/giva}}

\section{RELATED WORKS}
\paragraph{Parameter Efficient Fine-Tuning.}
Early PEFT methods introduce a small set of trainable parameters and keep the original model frozen during training. One line of work inserts adapter modules between different layers of a pre-trained model \citep{houlsby2019parameter, lin2020exploring, pfeiffer2021adapterfusion}, while other methods focus on adding learnable prompts to the model input. For example, AutoPrompt \citep{shin2020autoprompt} searches for a token sequence that elicits the desired behavior. Prompt Tuning \citep{lester2021power} and Prefix Tuning \citep{li2021prefix} extend this idea by learning continuous prefix embeddings rather than discrete tokens. Although effective, these methods increase the inference latency due to the additional layers or tokens. LoRA \citep{Hu2021LoRALA} addresses this limitation by modeling the parameter updates as a product of two low-rank matrices, which can be merged back into the base model after training, thus eliminating any additional inference cost. Several works have proposed alternative initialization schemes for LoRA \citep{buyukakyuz2024olora, yang2024corda, meng2024pissa}. Methods like AdaLoRA \citep{zhang2023adalora}, EVA \citep{paischer2024parameter}, GoRA \citep{he2025gora}, and GeLoRA \citep{ed2024gelora} decide the ranks for different layers in the model dynamically rather than using a fixed rank across layers. To achieve stable learning, LoRA+ \citep{hayou2024loraplus} uses different learning rates for the low-rank matrices, while \citet{zhang2024riemannian} propose the use of Riemannian preconditioning. Methods such as FLoRA \citep{hao_flora_2024} and XGBLoRA \citep{zhang2024less}, among others, have proposed full-rank variants of LoRA. Another line of research in parameter-efficient methods focuses on preserving the hyperspherical energy during fine-tuning. For example, Orthogonal Fine-Tuning \citep[OFT;][]{qiu2023controlling} targets text-to-image tasks. BOFT \citep{liu2024parameter} and MOFT \citep{wu2025memory} are its parameter- and memory-efficient extensions.


\paragraph{Improved Parameter Efficiency.}
Several methods have improved parameter efficiency by introducing learnable scaling vectors in LoRA-inspired frameworks. For example, NoLA \citep{koohpayegani2024nola} represents LoRA's low-rank matrices as a weighted combination of random bases, training only the weights. This strategy achieves comparable performance to LoRA while substantially reducing trainable parameters. VeRA \citep{kopiczko2024vera} employs trainable scaling vectors to adapt a pair of randomly initialized matrices to control the model behavior. Instead of using random bases, OSoRA \citep{Han2025OSoRAOA} derives its bases from pre-trained weights. RandLoRA \citep{albert2025randlora} extends this strategy to support full-rank updates while preserving parameter efficiency. Differing from these methods, LoRA-XS \citep{balazy2024lora} introduces a trainable weight matrix between the frozen low-rank matrices, which are derived from pre-trained weights.


\paragraph{Leveraging Gradients in PEFT.}
Recent studies have leveraged the full fine-tuning gradient for better initialization of the low-rank matrices. LoRA-GA \citep{wang2024lora} and LoRA-One \citep{zhang2025loraone} use the first-step full fine-tuning gradient for initialization and achieve faster convergence across tasks. LoRA-SB \citep{ponkshe2024initialization} extends LoRA-XS and uses gradients to initialize the frozen low-rank matrices. LoRA-Pro \citep{wang2025lorapro} adjusts the gradients of the low-rank matrices to better approximate those of full fine-tuning.

\section{METHOD}
\label{sec:method}
\subsection{Preliminaries}
For a pre-trained weight matrix $W_\pt \in \mathbb{R}^{m \times d}$, vector-based adaptation methods parameterize the weight update $\Delta W$ and the updated weight matrix $W^\prime \in \mathbb{R}^{m \times d}$ as
\begin{equation}
\label{eq:formulation}
\begin{split}
    \Delta W &= \ovector B \svector A \\
    W^\prime &= W_\pt + \Delta W = W_\pt + \ovector B \svector A
\end{split}
\end{equation}
Here, $B \in \mathbb{R}^{m \times r}$ and $A \in \mathbb{R}^{r \times d}$ are low-rank matrices; further $\ovector = \operatorname{diag}(\gamma_1 \dots \gamma_m)$ and $\svector = \operatorname{diag}(\lambda_1 \dots \lambda_r)$. Rank $r$ is generally chosen such that $r < \operatorname{min}(m, d)$. During training, only the diagonal matrices are updated, resulting in $m + r$ trainable parameters. Hence, these methods are extremely parameter-efficient.

VeRA \citep{kopiczko2024vera} initializes $A$ and $B$ as random matrices shared across layers, while OSoRA \citep{Han2025OSoRAOA} derives $A$ and $B$ from the singular-value decomposition (SVD) of $W_\pt$.

\subsection{Proposed Approach}
In this section, we introduce \vega{}. We hypothesize that carefully chosen $A$ and $B$ may allow for smaller values of $r$ than existing vector-based adaptation methods while achieving comparable performance. 

We employ a gradient-based initialization for $A$ and $B$. In our method, $A$ and $B$ are chosen such that the first-step update closely approximates the first-step full fine-tuning update.

\paragraph{First-Step Update.}
The weight update $\Delta W$ in \eqref{eq:formulation} can be written as
\begin{equation*}
    \begin{split}
    \Delta W &= (\ovector_0 + \Delta \ovector) B (\svector_0 + \Delta \svector) A \\
    &= \ovector_0 B \svector_0 A + \ovector_0 B \Delta \svector A + \Delta \ovector B \svector_0 A + \Delta \ovector B \Delta \svector A
    \end{split}
\end{equation*}
Here, a subscript of zero denotes the initial value of a variable. We use $\ovector_0 = \mathbf{0}^{m \times m}$ and $\svector_0 = \mathbb{I}^{r \times r}$, which reduces $\Delta W$ to
\begin{equation}
\label{eq:delta_w_simplified}
    \Delta W = \Delta \ovector B A + \Delta \ovector B \Delta \svector A
\end{equation}

The first-step update can be computed by substituting $\Delta \ovector = \eta \nabla_{\ovector} \mathcal{L} (\ovector_0)$ and $\Delta \svector = \eta \nabla_{\svector} \mathcal{L} (\svector_0)$. Here, $\eta$ and $\mathcal{L}$ denote the learning rate and the loss function, respectively. The second term in \eqref{eq:delta_w_simplified} contains two delta terms, resulting in a factor of $\eta^2$ and a product of two gradients. We ignore this term and rewrite the first-step update as
\begin{equation*}
    \eta \nabla_{\ovector} \mathcal{L} (\ovector_0) B A
\end{equation*}

\paragraph{Computing $A$ and $B$.}
Let $\eta \nabla_W \mathcal{L} (W_\pt)$ denote the first-step full fine-tuning update. We choose $A$ and $B$ such that 
\begin{equation}
    \label{eq:optim}
    \argmin_{A, B} \Vert \eta \nabla_{\ovector} \mathcal{L} (\ovector_0) B A - \eta \nabla_W \mathcal{L} (W_\pt) \Vert_F
\end{equation}

$\nabla_{\ovector} \mathcal{L}$ can be written in terms of $\nabla_{W^\prime} \mathcal{L}$, resulting in
\begin{equation*}
\nabla_{\ovector} \mathcal{L} (\ovector_0) = \nabla_{W^\prime} \mathcal{L} (W_\pt) A^T B^T
\end{equation*}

As $\nabla_{W^\prime} \mathcal{L} (W_\pt) = \nabla_{W} \mathcal{L} (W_\pt)$ \citep[Lemma~3.1 from][]{wang2024lora}, the optimization problem in \eqref{eq:optim} can be rewritten as
\begin{equation}
    \label{eq:optim_rewrite}
    \argmin_{A, B} \Vert \nabla_{W} \mathcal{L} (W_\pt) A^T B^T B A - \nabla_W \mathcal{L} (W_\pt) \Vert_F
\end{equation}

\begin{tcolorbox}
\begin{theorem}
\label{theorem:solution}
The optimization problem in \eqref{eq:optim_rewrite} reaches its minimum when
\begin{equation*}
    A = V_{r}^T, B^T B = \mathbb{I}^{r \times r}
\end{equation*}
$V_r$ denotes the right singular vectors of $\nabla_W \mathcal{L} (W_\pt)$, corresponding to the $r$ largest singular values.
\end{theorem}

\begin{proof}
See the supplementary material.
\end{proof}
\end{tcolorbox}

Using Theorem~\ref{theorem:solution}, we initialize the training formulation in \eqref{eq:formulation} as follows
\begin{equation}
\label{eq:init_sol}
\begin{split}
    &\ovector = \mathbf{0}^{m \times m}, \\
    &\svector = \mathbb{I}^{r \times r}, \\
    &A = V_{r}^T, \; \text{and}\\
    &B \; \text{such that} \; B^T B = \mathbb{I}^{r \times r}
\end{split}
\end{equation}

\paragraph{Efficiency Trade-Offs.}
In contrast to existing vector-based adaptation methods, our approach introduces some additional overheads: (1) storage, since matrices $A$ and $B$ must be saved, (2) gradient computation, and (3) performing SVD to obtain $A$ and $B$. The low-rank matrices need to be stored only once per task, and subsequent checkpoints remain as lightweight as those in existing vector-based adaptation methods. To keep the training pipeline efficient, we compute the first-step full fine-tuning gradient using a single batch of examples and use the low-rank SVD algorithm.\footnote{PyTorch implementation: \href{https://docs.pytorch.org/docs/stable/generated/torch.svd\_lowrank.html}{\texttt{torch.svd\_lowrank}}}

\section{EXPERIMENTS}
We conduct extensive experiments to evaluate \vega{} on natural language understanding (Section~\ref{sec:nlu}), natural language generation (Section~\ref{sec:nlg}), and image classification (Section~\ref{sec:image_classification}). The corresponding sections provide details about the models and datasets used. We conclude by discussing the training times for different adaptation methods in Section~\ref{sec:training_time_comparison}.

For our baselines, we primarily consider vector-based adaptation methods: VeRA \citep{kopiczko2024vera}, OSoRA \citep{Han2025OSoRAOA}, and RandLoRA \citep{albert2025randlora}. We additionally include LoRA \citep{Hu2021LoRALA}, which serves as a widely accepted reference point in PEFT methods. While not all baselines are evaluated across every dataset, we include each method for all tasks where results were reported in its original publication.

\begin{table*}[t]
    \caption{Performance of different adaptation methods on GLUE. We report Matthew's correlation for CoLA, Pearson's correlation for STS-B, and accuracy for the remaining tasks. Experiments are repeated three times, and the averages are reported.}
    \label{tab:glue_results}
    \begin{center}
    \begin{small}
    \begin{tabular}{lrccccccc}
    \toprule
    \multirow{2}{*}{\textbf{Method}} & \textbf{Params} & \textbf{SST-2} & \textbf{MRPC} & \textbf{CoLA} & \textbf{QNLI} & \textbf{RTE} & \textbf{STS-B} & \multirow{2}{*}{\textbf{Avg.}} \\
    {} & (M) & (Acc.) & (Acc.) & (MCC) & (Acc.) & (Acc.) & (PCC) & {} \\
    \midrule
    \multicolumn{9}{c}{\textit{RoBERTa-Base}} \\
    \midrule
    Full Fine-Tuning & 124.06 & 94.8{\tiny $\pm$0.2} & 89.6{\tiny $\pm$0.7} & 62.1{\tiny $\pm$0.8} & 92.8{\tiny $\pm$0.1} & 75.8{\tiny $\pm$1.4} & 90.9{\tiny $\pm$0.1} & 84.3 \\
    LoRA ($r = 4$) & 0.66 & 94.7{\tiny $\pm$0.1} & 87.7{\tiny $\pm$0.6} & 59.2{\tiny $\pm$0.8} & 92.7{\tiny $\pm$0.1} & 74.8{\tiny $\pm$0.6} & 90.5{\tiny $\pm$0.0} & 83.3 \\
    VeRA ($r = 1024$) & 0.16 & 94.5{\tiny $\pm$0.1} & 88.2{\tiny $\pm$0.9} & 62.3{\tiny $\pm$0.5} & 92.5{\tiny $\pm$0.1} & 73.6{\tiny $\pm$1.3} & 90.6{\tiny $\pm$0.1} & 83.6 \\
    RandLoRA ($r = 64$) & 0.72 & 94.9{\tiny $\pm$0.1} & 88.8{\tiny $\pm$0.4} & 62.6{\tiny $\pm$0.9} & 92.8{\tiny $\pm$0.2} & 74.8{\tiny $\pm$1.5} & 90.6{\tiny $\pm$0.1} & 84.1 \\
    \rowcolor{\vegacolor}
    \vega{} ($r = 8$, \initvrur{}) & 0.08 & 94.2{\tiny $\pm$0.2} & 88.7{\tiny $\pm$0.5} & 59.1{\tiny $\pm$0.8} & 92.3{\tiny $\pm$0.2} & 75.5{\tiny $\pm$3.9} & 90.4{\tiny $\pm$0.1} & 83.4 \\
    \rowcolor{\vegacolor}
    \vega{} ($r = 8$, \initvrutwor{}) & 0.08 & 94.2{\tiny $\pm$0.2} & 89.4{\tiny $\pm$0.9} & 60.0{\tiny $\pm$2.3} & 92.2{\tiny $\pm$0.2} & 73.4{\tiny $\pm$1.9} & 90.5{\tiny $\pm$0.0} & 83.3 \\
    \rowcolor{\vegacolor}
    \vega{} ($r = 8$, \initvr{}) & 0.08 & 94.5{\tiny $\pm$0.2} & 89.3{\tiny $\pm$1.6} & 61.6{\tiny $\pm$0.3} & 92.3{\tiny $\pm$0.1} & 74.8{\tiny $\pm$3.0} & 90.5{\tiny $\pm$0.2} & 83.8 \\
    \midrule
    \multicolumn{9}{c}{\textit{RoBERTa-Large}} \\
    \midrule
    Full Fine-Tuning & 354.31 & 96.4{\tiny $\pm$0.1} & 89.5{\tiny $\pm$0.6} & 67.5{\tiny $\pm$1.5} & 94.7{\tiny $\pm$0.1} & 84.7{\tiny $\pm$1.8} & 92.1{\tiny $\pm$0.1} & 87.5 \\
    LoRA ($r = 4$) & 1.77 & 96.2{\tiny $\pm$0.2} & 89.7{\tiny $\pm$0.4} & 66.6{\tiny $\pm$0.5} & 94.7{\tiny $\pm$0.0} & 85.2{\tiny $\pm$1.0} & 92.2{\tiny $\pm$0.1} & 87.4 \\
    VeRA ($r = 256$) & 0.26 & 96.1{\tiny $\pm$0.1} & 89.8{\tiny $\pm$0.4} & 68.6{\tiny $\pm$1.8} & 94.8{\tiny $\pm$0.2} & 83.2{\tiny $\pm$2.2} & 92.0{\tiny $\pm$0.2} & 87.4 \\
    RandLoRA ($r = 100$) & 1.78 & 96.5{\tiny $\pm$0.5} & 90.1{\tiny $\pm$0.4} & 67.7{\tiny $\pm$1.4} & 94.7{\tiny $\pm$0.3} & 85.3{\tiny $\pm$0.9} & 92.3{\tiny $\pm$0.2} & 87.8 \\
    \rowcolor{\vegacolor}
    \vega{} ($r = 8$, \initvrur{}) & 0.22 & 96.2{\tiny $\pm$0.2} & 90.3{\tiny $\pm$0.7} & 67.5{\tiny $\pm$1.0} & 94.7{\tiny $\pm$0.1} & 84.5{\tiny $\pm$1.1} & 91.6{\tiny $\pm$0.5} & 87.5 \\
    \rowcolor{\vegacolor}
    \vega{} ($r = 8$, \initvrutwor{}) & 0.22 & 96.1{\tiny $\pm$0.4} & 89.4{\tiny $\pm$0.4} & 68.0{\tiny $\pm$2.2} & 95.0{\tiny $\pm$0.1} & 82.3{\tiny $\pm$2.8} & 91.6{\tiny $\pm$0.2} & 87.1 \\
    \rowcolor{\vegacolor}
    \vega{} ($r = 8$, \initvr{}) & 0.22 & 96.1{\tiny $\pm$0.3} & 89.8{\tiny $\pm$1.2} & 67.0{\tiny $\pm$0.9} & 94.7{\tiny $\pm$0.2} & 85.3{\tiny $\pm$0.6} & 92.0{\tiny $\pm$0.4} & 87.5 \\
    \bottomrule
    \end{tabular}
    \end{small}
    \end{center}
\end{table*}
\begin{table*}[t]
    \caption{Performance of different adaptation methods on commonsense reasoning datasets. We report accuracy for all tasks. The percentages are computed relative to the total number of parameters. Experiments are repeated three times, and their averages are reported.}
    \label{tab:commonsense_results}
    \vspace{0.5em}
    \setlength{\tabcolsep}{3pt}
    \begin{center}
    \begin{small}
    \begin{tabular}{lcccccccccc}
    \toprule
    \multirow{2}{*}{\textbf{Method}} & \textbf{Params} & \textbf{BoolQ} & \textbf{PIQA} & \textbf{SIQA} & \textbf{HS} & \textbf{WG} & \textbf{ARC-e} & \textbf{ARC-c} & \textbf{OBQA} & \multirow{2}{*}{\textbf{Avg.}} \\
    {} & (\%) & \multicolumn{8}{c}{(Acc.)} & {} \\
    \midrule
    \multicolumn{11}{c}{\textit{Qwen2-0.5B-Instruct}} \\
    \midrule
    LoRA ($r = 16$) & 1.180 & 62.0{\tiny $\pm$0.3} & 66.1{\tiny $\pm$1.6} & 61.3{\tiny $\pm$0.1} & 54.0{\tiny $\pm$2.1} & 51.2{\tiny $\pm$1.8} & 66.6{\tiny $\pm$0.2} & 49.1{\tiny $\pm$2.3} & 60.7{\tiny $\pm$1.0} & 58.9 \\
    VeRA ($r = 1024$) & 0.058 & 62.4{\tiny $\pm$0.4} & 65.5{\tiny $\pm$0.2} & 60.5{\tiny $\pm$0.4} & 47.6{\tiny $\pm$1.0} & 53.3{\tiny $\pm$1.3} & 64.6{\tiny $\pm$0.5} & 45.7{\tiny $\pm$0.7} & 57.1{\tiny $\pm$1.4} & 57.1 \\
    OSoRA ($r = 1024$) & 0.048 & 62.0{\tiny $\pm$0.3} & 62.0{\tiny $\pm$0.6} & 58.5{\tiny $\pm$1.0} & 40.4{\tiny $\pm$1.3} & 50.9{\tiny $\pm$0.8} & 62.7{\tiny $\pm$0.6} & 43.9{\tiny $\pm$0.8} & 57.3{\tiny $\pm$1.0} & 54.7 \\
    RandLoRA ($r = 10$) & 1.191 & 61.4{\tiny $\pm$1.8} & 65.6{\tiny $\pm$1.7} & 57.6{\tiny $\pm$1.9} & 44.1{\tiny $\pm$5.5} & 50.2{\tiny $\pm$0.4} & 64.6{\tiny $\pm$0.9} & 45.9{\tiny $\pm$2.9} & 55.8{\tiny $\pm$2.8} & 55.7 \\
    \rowcolor{\vegacolor}
    \vega{} ($r = 64$, \initvrur{}) & 0.035 & 62.6{\tiny $\pm$0.3} & 66.2{\tiny $\pm$0.7} & 59.3{\tiny $\pm$0.2} & 48.8{\tiny $\pm$0.6} & 51.6{\tiny $\pm$0.5} & 66.4{\tiny $\pm$0.1} & 47.6{\tiny $\pm$1.6} & 54.7{\tiny $\pm$0.4} & 57.1 \\
    \rowcolor{\vegacolor}
    \vega{} ($r = 64$, \initvrutwor{}) & 0.035 & 61.7{\tiny $\pm$1.0} & 66.0{\tiny $\pm$0.8} & 60.2{\tiny $\pm$1.0} & 50.4{\tiny $\pm$0.7} & 50.6{\tiny $\pm$0.5} & 65.8{\tiny $\pm$0.6} & 48.5{\tiny $\pm$0.3} & 57.8{\tiny $\pm$1.0} & 57.6 \\
    \rowcolor{\vegacolor}
    \vega{} ($r = 64$, \initvr{}) & 0.035 & 62.2{\tiny $\pm$0.0} & 65.6{\tiny $\pm$0.7} & 60.4{\tiny $\pm$0.4} & 50.7{\tiny $\pm$1.3} & 51.9{\tiny $\pm$0.5} & 65.7{\tiny $\pm$1.0} & 48.7{\tiny $\pm$0.5} & 55.7{\tiny $\pm$0.9} & 57.6 \\
    \midrule
    \multicolumn{11}{c}{\textit{Phi3-Mini}} \\
    \midrule
    LoRA ($r = 16$) & 0.573 & 69.9{\tiny $\pm$0.6} & 86.4{\tiny $\pm$0.1} & 77.2{\tiny $\pm$0.5} & 90.4{\tiny $\pm$0.5} & 80.4{\tiny $\pm$0.3} & 96.1{\tiny $\pm$0.4} & 88.1{\tiny $\pm$0.5} & 87.7{\tiny $\pm$0.6} & 84.5 \\
    VeRA ($r = 1024$) & 0.027 & 69.2{\tiny $\pm$1.0} & 85.8{\tiny $\pm$0.2} & 76.5{\tiny $\pm$0.1} & 89.1{\tiny $\pm$1.7} & 79.4{\tiny $\pm$1.2} & 95.5{\tiny $\pm$0.4} & 86.9{\tiny $\pm$0.3} & 86.1{\tiny $\pm$1.6} & 83.6 \\
    OSoRA ($r = 1024$) & 0.027 & 68.4{\tiny $\pm$0.5} & 85.8{\tiny $\pm$0.5} & 77.8{\tiny $\pm$0.2} & 88.7{\tiny $\pm$0.4} & 78.2{\tiny $\pm$0.1} & 95.7{\tiny $\pm$0.2} & 87.1{\tiny $\pm$0.3} & 88.9{\tiny $\pm$1.2} & 83.8 \\
    RandLoRA ($r = 40$) & 0.598 & 69.8{\tiny $\pm$0.1} & 86.7{\tiny $\pm$0.3} & 76.5{\tiny $\pm$0.6} & 90.1{\tiny $\pm$1.0} & 80.1{\tiny $\pm$0.2} & 95.6{\tiny $\pm$0.3} & 87.7{\tiny $\pm$0.3} & 87.1{\tiny $\pm$0.8} & 84.2 \\
    \rowcolor{\vegacolor}
    \vega{} ($r = 64$, \initvrur{}) & 0.024 & 68.5{\tiny $\pm$0.8} & 87.0{\tiny $\pm$0.3} & 76.3{\tiny $\pm$1.0} & 90.1{\tiny $\pm$0.8} & 78.8{\tiny $\pm$0.6} & 95.6{\tiny $\pm$0.1} & 87.6{\tiny $\pm$0.5} & 88.1{\tiny $\pm$1.6} & 84.0 \\
    \rowcolor{\vegacolor}
    \vega{} ($r = 64$, \initvrutwor{}) & 0.024 & 69.4{\tiny $\pm$0.2} & 85.4{\tiny $\pm$0.3} & 76.8{\tiny $\pm$0.1} & 89.7{\tiny $\pm$0.9} & 79.0{\tiny $\pm$0.9} & 95.3{\tiny $\pm$0.2} & 87.1{\tiny $\pm$0.7} & 87.5{\tiny $\pm$0.1} & 83.8 \\
    \rowcolor{\vegacolor}
    \vega{} ($r = 64$, \initvr{}) & 0.024 & 68.5{\tiny $\pm$1.1} & 86.3{\tiny $\pm$0.5} & 76.0{\tiny $\pm$1.1} & 89.5{\tiny $\pm$1.5} & 78.0{\tiny $\pm$1.2} & 95.5{\tiny $\pm$0.2} & 87.4{\tiny $\pm$0.3} & 86.7{\tiny $\pm$1.8} & 83.5 \\
    \midrule
    \multicolumn{11}{c}{\textit{OLMo2-7B}} \\
    \midrule
    LoRA ($r = 16$) & 0.383 & 71.0{\tiny $\pm$0.2} & 86.7{\tiny $\pm$0.1} & 78.4{\tiny $\pm$0.6} & 92.3{\tiny $\pm$0.4} & 80.0{\tiny $\pm$0.2} & 93.0{\tiny $\pm$0.2} & 83.0{\tiny $\pm$0.1} & 85.3{\tiny $\pm$0.5} & 83.7 \\
    VeRA ($r = 1024$) & 0.014 & 70.5{\tiny $\pm$0.8} & 85.8{\tiny $\pm$0.6} & 77.7{\tiny $\pm$0.4} & 91.0{\tiny $\pm$0.2} & 77.8{\tiny $\pm$0.6} & 92.7{\tiny $\pm$0.1} & 81.7{\tiny $\pm$0.5} & 85.3{\tiny $\pm$0.9} & 82.8 \\
    OSoRA ($r = 1024$) & 0.014 & 69.4{\tiny $\pm$1.1} & 85.7{\tiny $\pm$0.9} & 78.1{\tiny $\pm$0.7} & 90.8{\tiny $\pm$1.1} & 78.2{\tiny $\pm$1.3} & 92.0{\tiny $\pm$0.2} & 81.3{\tiny $\pm$0.4} & 85.3{\tiny $\pm$1.3} & 82.6 \\
    RandLoRA ($r = 100$) & 0.376 & 70.1{\tiny $\pm$1.6} & 86.3{\tiny $\pm$0.2} & 77.5{\tiny $\pm$0.8} & 89.7{\tiny $\pm$2.8} & 78.3{\tiny $\pm$2.9} & 92.8{\tiny $\pm$0.2} & 83.7{\tiny $\pm$0.1} & 85.9{\tiny $\pm$0.6} & 83.0 \\
    \rowcolor{\vegacolor}
    \vega{} ($r = 64$, \initvrur{}) & 0.012 & 71.0{\tiny $\pm$0.4} & 86.8{\tiny $\pm$0.3} & 78.3{\tiny $\pm$0.6} & 91.2{\tiny $\pm$0.6} & 78.7{\tiny $\pm$0.9} & 93.2{\tiny $\pm$0.4} & 82.5{\tiny $\pm$0.4} & 84.1{\tiny $\pm$0.9} & 83.2 \\
    \rowcolor{\vegacolor}
    \vega{} ($r = 64$, \initvrutwor{}) & 0.012 & 70.8{\tiny $\pm$1.1} & 86.1{\tiny $\pm$0.9} & 79.0{\tiny $\pm$0.6} & 91.3{\tiny $\pm$1.4} & 78.5{\tiny $\pm$1.3} & 93.1{\tiny $\pm$0.2} & 82.8{\tiny $\pm$0.5} & 84.7{\tiny $\pm$0.6} & 83.3 \\
    \rowcolor{\vegacolor}
    \vega{} ($r = 64$, \initvr{}) & 0.012 & 71.2{\tiny $\pm$0.5} & 86.7{\tiny $\pm$0.4} & 78.1{\tiny $\pm$1.1} & 91.0{\tiny $\pm$0.6} & 78.7{\tiny $\pm$0.7} & 92.7{\tiny $\pm$0.4} & 82.5{\tiny $\pm$0.7} & 84.4{\tiny $\pm$2.3} & 83.2 \\
    \bottomrule
    \end{tabular}
    \end{small}
    \end{center}
\end{table*}

We initialize $A$ following \eqref{eq:init_sol}. As $B$ may be initialized as any matrix with orthonormal columns, we consider the following initialization strategies:
\vspace{-0.75em}
\paragraph{\initvrur[\small].}
Left singular vectors of the first-step full fine-tuning gradient corresponding to the $r$ largest singular values. \citet{zhang2025loraone} use a similar strategy for LoRA.

\vspace{-0.75em}
\paragraph{\initvrutwor[\small].}
Left singular vectors of the first-step full fine-tuning gradient corresponding to the $r + 1$ to $2r$ largest singular values. \citet{wang2024lora} use a similar strategy for LoRA.

\vspace{-0.75em}
\paragraph{\initvr[\small].}
By performing a QR decomposition on a random matrix $X \sim \mathcal{N}(0, 1)^{m \times r}$.

\subsection{Natural Language Understanding}
\label{sec:nlu}

We begin by evaluating \vega{} on the General Language Understanding Evaluation (GLUE) benchmark \citep{wang2019glue}. Our experiments use both the base and large versions of RoBERTa \citep{liu2019roberta}. We generally follow the training setup of \citet{albert2025randlora} and additionally perform a learning rate sweep. Adapter modules are added to all linear projections within the attention and MLP layers. We initialize the training from pre-trained checkpoints and omit the MNLI-trick from \citet{Hu2021LoRALA} due to resource limitations.

Models are fine-tuned for 10 epochs and evaluated at the end of each epoch. We report the best validation performance. Further details on the training setup are provided in Section~\ref{sec:glue_training_details} of the supplementary material.

Table~\ref{tab:glue_results} shows the performance of various adaptation methods on the validation set. The reported number of trainable parameters does not include the classification heads. These results demonstrate that \vega{} achieves performance comparable to other methods. Notably, \vega{} outperforms VeRA and uses a rank of eight, substantially reducing the computational overhead during training. All three initializations for \vega{} yield comparable results on GLUE, but \initvr[\small] achieves the highest average performance across model sizes.

\subsection{Natural Language Generation}
\label{sec:nlg}
We next evaluate \vega{} on natural language generation tasks. For this purpose, we consider four benchmarks: (1) commonsense reasoning, (2) mathematical reasoning, (3) code generation, and (4) instruction tuning.

\paragraph{Commonsense Reasoning.}
\label{sec:commonsense_reasoning}
We fine-tune \qwen{} \citep[0.5B;][]{team2024qwen2}, \msphi{} \citep[3.8B;][]{Abdin2024Phi3TR}, and \olmo{} \citep[7B;][]{olmo20242} on a subset of 15K commonsense reasoning examples from \citet{hu2023llm}. We largely follow the training setup of \citet{albert2025randlora}, with the addition of a learning rate sweep. The official implementation of RandLoRA suggests that \citet{albert2025randlora} compute the training loss over both the input and output tokens. In our experiments, we restrict the loss computation to the output tokens, which yields better performance across methods. Adapter modules are added to the attention layers' query, key, and value projections, as well as the up and down projections of the MLP layers. For \msphi{}, they are also added to the MLP gate projections.

We fine-tune the models for three epochs. During training, the models are evaluated four times per epoch, and the best checkpoint is selected based on the validation loss. See Section~\ref{sec:csr_training_details} of the supplementary material for further details on the training setup.

The fine-tuned models are evaluated on the following commonsense reasoning datasets: (1) BoolQ \citep{clark2019boolq}, (2) PIQA \citep{bisk2020piqa}, (3) Social IQa \citep{sap-etal-2019-social}, (4) HellaSwag \citep{zellers2019hellaswag}, (5) WinoGrande \citep{sakaguchi2020winogrande}, (6) ARC \citep{clark2018arc}, and (7) OpenBookQA \citep{mihaylov2018can}. We use beam search with a beam size of four and generate a maximum of 32 tokens.

Table~\ref{tab:commonsense_results} presents the performance of various adaptation methods. The parameter percentages are shown relative to the total model size. These results demonstrate that \vega{} achieves performance comparable to other methods and surpasses both VeRA and OSoRA despite using a rank of 64 across models. In addition, the three initialization strategies for \vega{} yield broadly comparable results on commonsense reasoning tasks.

Unlike \citet{albert2025randlora}, we do not observe a significant performance gap between VeRA and other methods. We believe the larger performance gap observed by \citet{albert2025randlora} for VeRA can be explained by the choice of the $d_\text{initial}$ hyperparameter, in addition to the difference in the training loss computation.

\vspace{-0.5em}

\paragraph{Mathematical Reasoning.}
\label{sec:mathematical_reasoning}
We fine-tune \olmo{} (7B) on 100K examples from MetaMathQA \citep{yu2024metamath} and use an additional 10K examples for validation. Adapter modules are added to the attention layers' query, key, and value projections, as well as the up and down projections of the MLP layers. The models are fine-tuned for one epoch. During training, the validation is performed 10 times, and the best checkpoint is selected based on the validation loss. Further details on the training setup are provided in Section~\ref{sec:math_reasoning_training_details} of the supplementary material.

The fine-tuned models are evaluated on the GSM8k test set \citep{cobbe2021gsm8k}. We use greedy decoding and generate a maximum of 256 tokens. Accuracy is computed by matching the final answer with the ground truth.

Results for these experiments are presented in Table~\ref{tab:metamath_code_results}. These results show that \vega{} surpasses both VeRA and OSoRA despite using a rank of only 64. Similar to earlier experiments, the three initializations for \vega{} achieve comparable performance, with \initvr[\small] performing the best.

\paragraph{Code Generation.}
We fine-tune \olmo{}~(7B) on 100K examples from Code-Feedback \citep{zheng2024opencodeinterpreter} and use an additional 10K examples for validation. We use the same training setup as used for mathematical reasoning experiments.

The fine-tuned models are evaluated on HumanEval \citep{chen2021codex}. We use beam search with a beam size of four and generate a maximum of 256 tokens. For scoring the model generations, we use the official implementation.\footnote{\url{https://github.com/openai/human-eval}}

Table~\ref{tab:metamath_code_results} shows the performance of various adaptation methods. The results indicate that \vega{} matches the performance of VeRA and OSoRA while requiring substantially lower ranks. For code generation, \initvr[\small] achieves the strongest results, while the other two initializations outperform VeRA.  Additionally, we observe that LoRA outperforms vector-based adaptation methods by a considerable margin.

\paragraph{Instruction Tuning.}
Instruction tuning adapts pre-trained language models to follow instructions more effectively. For our experiments, we fine-tune \mistral{} \citep[7B;][]{jiang2023mistral7b} on the Alpaca dataset \citep{alpaca}. We largely follow the training setup of \citet{kopiczko2024vera} and additionally perform a learning rate sweep. We use a single random seed for all methods. Adapter modules are added to all linear projections within the attention and MLP layers. Models are fine-tuned for one epoch without any validation, and the model at the end of the epoch is used for evaluation. For \vega{}, we only use \initvr[\small], as it is the most generic initialization out of the three strategies. Due to the costs associated with evaluation, we use the GPT-4 reference answers from MT-Bench and compute the loss on these answers to select the best learning rate for each method. Further details on the training setup are provided in Section~\ref{sec:instruction_tuning_training_details} of the supplementary material.

\begin{table}[t]
    \caption{Performance of fine-tuned \olmo{} (7B) on mathematical reasoning and code generation with different adaptation methods. We report accuracy for GSM8k and pass@1 for HumanEval. Experiments are repeated three times, and the averages are reported.}
    \label{tab:metamath_code_results}
    \setlength{\tabcolsep}{3pt}
    \newcolumntype{Z}{>{\setlength{\tabcolsep}{1pt}}r}
    \begin{center}
    \begin{small}
    \begin{tabular}{lZcc}
    \toprule
    \multirow{2}{*}{\textbf{Method}} & \textbf{Params} & \textbf{GSM8k} & \textbf{HEval} \\
    {} & (M) & (Acc.) & (pass@1) \\
    \midrule
    LoRA ($r = 8$) & 14.02 & 70.9{\tiny $\pm$0.7} & 27.6{\tiny $\pm$2.3} \\
    VeRA ($r = 512$) & 0.96 & 70.4{\tiny $\pm$0.8} & 20.1{\tiny $\pm$1.8} \\
    OSoRA ($r = 512$) & 0.96 & 69.7{\tiny $\pm$0.6} & 23.4{\tiny $\pm$0.9} \\
    \rowcolor{Gray!30}
    \vega{} ($r = 64$, \initvrur{}) & 0.89 & 70.3{\tiny $\pm$0.9} & 22.0{\tiny $\pm$1.6} \\
    \rowcolor{Gray!30}
    \vega{} ($r = 64$, \initvrutwor{}) & 0.89 & 70.6{\tiny $\pm$0.4} & 22.2{\tiny $\pm$1.3} \\
    \rowcolor{Gray!30}
    \vega{} ($r = 64$, \initvr{}) & 0.89 & 70.8{\tiny $\pm$0.8} & 23.2{\tiny $\pm$0.6} \\
    \bottomrule
    \end{tabular}
    \end{small}
    \end{center}
\end{table}

We test the fine-tuned models on MT-Bench \citep{zheng2023judging}. The best model for each method is evaluated using GPT-4.\footnote{\texttt{gpt-4-0613} as of Nov 25, 2025.}

Table~\ref{tab:mtbench} presents the performance of different adaptation methods. These results demonstrate that \vega{} achieves performance comparable to both LoRA and VeRA despite using a rank of 32 and even outperforms both on the first turn questions.

\subsection{Image Classification}
\label{sec:image_classification}
Our final evaluation focuses on image classification. We adapt the vision backbones of DinoV2 ViT-B/14 \citep{oquab2023dinov2} and CLIP ViT-L/14 \citep{radford2021learning} across four image classification datasets: (1) CIFAR100 \citep{krizhevsky2009learning}, (2) Food101 \citep{bossard2014food}, (3) Flowers102 \citep{nilsback2008automated}, and (4) RESISC45 \citep{cheng2017remote}. Adapter modules are added to the attention layers' query, key, and value projections, as well as the linear projections in the MLP layers. We largely follow the training setup of \citet{albert2025randlora} and additionally perform a learning rate sweep. Further details on the training setup are provided in Section~\ref{sec:img_class_training_details} of the supplementary material.

Table~\ref{tab:vision_results} reports the results across different adaptation methods. Consistent with earlier experiments, \vega{} maintains competitive performance with other adaptation methods despite using a rank of 32. Across all datasets and models, the three initialization strategies for \vega{} achieve broadly similar performance.
\begin{table}[t]
    \setlength{\tabcolsep}{4pt}
    \caption{Average scores on MT-Bench given by GPT-4 to the responses from \mistral{} (7B) fine-tuned on the Alpaca dataset. For \vega{}, \initvr[\small] is used.}
    \label{tab:mtbench}
    \begin{center}
    \begin{small}
    \begin{tabular}{lrccc}
    \toprule
    \multirow{2}{*}{\textbf{Method}} & \textbf{Params} & \multicolumn{3}{c}{\textbf{Score}} \\
    {} & (M) & Turn 1 & Turn 2 & Avg. \\
    \midrule
    Pre-trained & - & 3.825 & 2.247 & 3.051 \\
    LoRA ($r = 64$) & 167.8 & 6.906 & 5.700 & 6.303 \\
    VeRA ($r = 1024$) & 1.6 & 6.956 & 5.375 & 6.166 \\
    \rowcolor{\vegacolor}
    \vega{} ($r = 32$) & 1.4 & 6.963 & 5.263 & 6.113 \\
    \bottomrule
    \end{tabular}
    \end{small}
    \end{center}
\end{table}
\begin{table*}[t]
    \caption{Performance of various adaptation methods on image classification datasets. We report accuracy for all tasks. Experiments are repeated three times, and their averages are reported.}
    \label{tab:vision_results}
    \begin{center}
    \begin{small}
    \begin{tabular}{lrcccc}
    \toprule
    \multirow{2}{*}{\textbf{Method}} & \textbf{Params} & \textbf{CIFAR100} & \textbf{Food101} & \textbf{Flowers102} & \textbf{RESISC45} \\
    {} & (M) & \multicolumn{4}{c}{(Acc.)} \\
    \midrule
    \multicolumn{6}{c}{\textit{DINOv2-ViT-B/14}} \\
    \midrule
    Full Fine-Tuning & 86.73 & 93.6{\tiny $\pm$0.2} & 93.9{\tiny $\pm$0.0} & 99.1{\tiny $\pm$0.2} & 97.2{\tiny $\pm$0.2} \\
    LoRA ($r = 32$) & 4.72 & 93.7{\tiny $\pm$0.1} & 93.8{\tiny $\pm$0.0} & 99.6{\tiny $\pm$0.0} & 96.6{\tiny $\pm$0.2} \\
    VeRA ($r = 256$) & 0.09 & 93.0{\tiny $\pm$0.3} & 93.2{\tiny $\pm$0.1} & 99.3{\tiny $\pm$0.1} & 96.0{\tiny $\pm$0.1} \\
    RandLoRA ($r = 10$) & 3.59 & 93.7{\tiny $\pm$0.1} & 94.1{\tiny $\pm$0.0} & 99.6{\tiny $\pm$0.0} & 96.9{\tiny $\pm$0.2} \\
    \rowcolor{\vegacolor}
    \vega{} ($r = 32$, \initvrur{}) & 0.08 & 93.1{\tiny $\pm$0.3} & 93.1{\tiny $\pm$0.1} & 99.2{\tiny $\pm$0.1} & 96.2{\tiny $\pm$0.0} \\
    \rowcolor{\vegacolor}
    \vega{} ($r = 32$, \initvrutwor{}) & 0.08 & 93.0{\tiny $\pm$0.1} & 93.1{\tiny $\pm$0.1} & 99.3{\tiny $\pm$0.1} & 95.9{\tiny $\pm$0.1} \\
    \rowcolor{\vegacolor}
    \vega{} ($r = 32$, \initvr{}) & 0.08 & 93.0{\tiny $\pm$0.1} & 93.2{\tiny $\pm$0.1} & 99.2{\tiny $\pm$0.2} & 95.8{\tiny $\pm$0.1} \\
    \midrule
    \multicolumn{6}{c}{\textit{CLIP-ViT-L/14}} \\
    \midrule
    Full Fine-Tuning & 303.28 & 92.8{\tiny $\pm$0.2} & 95.0{\tiny $\pm$0.1} & 98.3{\tiny $\pm$0.2} & 97.3{\tiny $\pm$0.1} \\
    LoRA ($r = 32$) & 12.58 & 93.1{\tiny $\pm$0.1} & 95.2{\tiny $\pm$0.2} & 98.3{\tiny $\pm$0.1} & 97.2{\tiny $\pm$0.1} \\
    VeRA ($r = 256$) & 0.23 & 92.4{\tiny $\pm$0.1} & 95.4{\tiny $\pm$0.1} & 98.7{\tiny $\pm$0.0} & 96.9{\tiny $\pm$0.1} \\
    RandLoRA ($r = 10$) & 12.78 & 93.0{\tiny $\pm$0.1} & 95.4{\tiny $\pm$0.2} & 98.2{\tiny $\pm$0.3} & 97.4{\tiny $\pm$0.1} \\
    \rowcolor{\vegacolor}
    \vega{} ($r = 32$, \initvrur{}) & 0.20 & 92.4{\tiny $\pm$0.2} & 95.4{\tiny $\pm$0.0} & 98.4{\tiny $\pm$0.2} & 96.9{\tiny $\pm$0.3} \\
    \rowcolor{\vegacolor}
    \vega{} ($r = 32$, \initvrutwor{}) & 0.20 & 92.3{\tiny $\pm$0.2} & 95.4{\tiny $\pm$0.1} & 98.5{\tiny $\pm$0.2} & 97.0{\tiny $\pm$0.1} \\
    \rowcolor{\vegacolor}
    \vega{} ($r = 32$, \initvr{}) & 0.20 & 92.4{\tiny $\pm$0.2} & 95.3{\tiny $\pm$0.0} & 98.7{\tiny $\pm$0.2} & 97.0{\tiny $\pm$0.1} \\
    \bottomrule
    \end{tabular}
    \end{small}
    \end{center}
\end{table*}

\subsection{Training Time Comparison}
\label{sec:training_time_comparison}
In this section, we quantify the training time of different adaptation methods. The time excludes the duration of validation loops. We report the average times for fine-tuning \qwen{} (0.5B) on Commonsense-15K, \olmo{} (7B) on MetaMathQA, and DINOv2-ViT-B/14 (87M) and CLIP-ViT-L/14 (303M) on the four image classification datasets.

Table~\ref{tab:time_comparison} presents the training times for various adaptation methods across models of different sizes. The results show that while \vega{} is only slightly better than the other methods on DINOv2-ViT-B/14, it gradually closes the gap with LoRA as the model size increases, proving to be significantly faster than the other vector-based adaptation methods.

\section{DISCUSSION}
In our experiments, we estimate the first-step full fine-tuning gradient using a single batch of examples. The results indicate that this approach works well across a range of tasks, and using different sets of examples for the first-step full fine-tuning gradient computation, controlled by varying the seed, has little effect on performance. Furthermore, the three initialization strategies for \vega{} perform comparably across tasks and models, providing further support for the solution derived in \eqref{eq:init_sol}.

We also investigate whether incorporating more batches into the gradient estimation and taking an average improves the model performance. We repeat the commonsense reasoning experiments for \vega{} with \initvrutwor[\small] and estimate the gradient by averaging over multiple batches. The training setup is the same as in Section~\ref{sec:commonsense_reasoning}. Figure~\ref{fig:grad_steps_vs_perf} shows the average performance for three models with different numbers of examples used for gradient estimation. These results suggest that incorporating more batches in the gradient computation does not improve model performance.

\begin{table}[t]
    \setlength{\tabcolsep}{3pt}
    \caption{Time taken for fine-tuning \qwen{} on Commonsense-15K, \olmo{} on MetaMathQA, and DINOv2-ViT-B/14 and CLIP-ViT-L/14 on the four image classification datasets. For each model, the training configuration is consistent across methods.}
    \label{tab:time_comparison}
    \begin{center}
    \begin{small}
    \begin{tabular}{lcccc}
    \toprule
    \textbf{Method} & \textbf{DINOv2} & \textbf{CLIP} & \textbf{\qwen{}} & \textbf{\olmo{}}  \\
    {} & (87M) & (303M) & (494M) & (7B) \\
    \midrule
    LoRA & 1h21m & 4h59m & 11m & 2h54m \\
    VeRA & 1h55m & 7h51m & 27m & 3h26m \\
    OSoRA & - & - & 23m & 3h25m \\
    RandLoRA & 2h05m & 7h55m & 23m & - \\
    \rowcolor{\vegacolor}
    \vega{} & 1h49m & 6h53m & 12m & 2h57m \\
    \bottomrule
    \end{tabular}
    \end{small}
    \end{center}
\end{table}
\begin{figure}[t]
    \centering
    \resizebox{0.9\linewidth}{!}{
    \fontfamily{phv}\selectfont\Large
    \includesvg{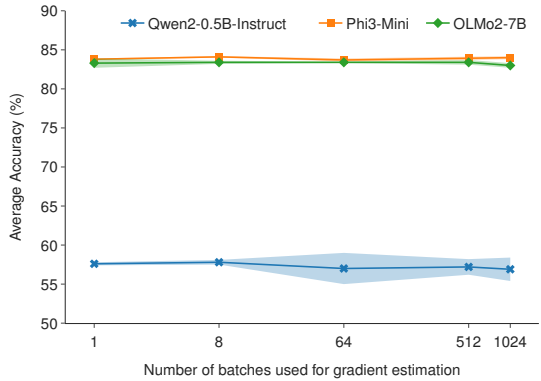}
    }
    \caption{Average commonsense reasoning performance for \vega{} (\initvrutwor[\small]) when multiple batches of examples are used to estimate the first-step full fine-tuning gradient.}
    \label{fig:grad_steps_vs_perf}
\end{figure}
\section{LIMITATIONS}
In this section, we note the limitations and potential extensions of our work. Although our method performs well on the datasets evaluated in this paper, its reliance on fixed bases computed from a single batch of examples may not fully capture the diversity of datasets composed of mixtures of multiple supervised fine-tuning sources, for example, Tulu 3 SFT mixture \citep{lambert2024tulu}, and may require further adaptation. In such cases, more sophisticated initialization strategies, such as incorporating gradients from multiple tasks within the mixture or combining random and gradient-based bases, could prove beneficial. Strategies for selecting the most representative examples for a dataset could improve performance across benchmarks. We leave these avenues for future work.

\section{CONCLUSION}
In this work, we introduced \vega{}, a parameter-efficient adaptation method for large language models. Unlike previous vector-based approaches, \vega{} constructs its bases from the first-step full fine-tuning gradient. Because these bases are derived from the fine-tuning data, \vega{} reduces the rank requirements needed to match the performance of existing vector-based adaptation methods. This reduction in required ranks leads to faster training times comparable to LoRA-like methods while preserving the extreme parameter efficiency of vector-based adaptation. We evaluated the proposed method across tasks in natural language understanding, natural language generation, and image classification, as well as on models of varying sizes. The results demonstrate that \vega{} outperforms or matches existing parameter-efficient methods while using significantly fewer trainable parameters and reducing rank requirements by a factor of eight compared to prior vector-based adaptation methods.

\subsubsection*{Acknowledgements}
This work was supported in part by the Illinois Campus Cluster Program and the Delta Advanced Computing and Data Resource. Delta is supported by the National Science Foundation (award OAC 2005572) and the State of Illinois and is a joint effort of the University of Illinois Urbana-Champaign and its National Center for Supercomputing Applications. We also thank the anonymous reviewers for their feedback.

\bibliographystyle{apalike}
\bibliography{custom}

\section*{Checklist}

\begin{enumerate}

  \item For all models and algorithms presented, check if you include:
  \begin{enumerate}
    \item A clear description of the mathematical setting, assumptions, algorithm, and/or model. [Yes, Section~\ref{sec:method}]
    \item An analysis of the properties and complexity (time, space, sample size) of any algorithm. [Yes, Section~\ref{sec:training_time_comparison}]
    \item (Optional) Anonymized source code, with specification of all dependencies, including external libraries. [No] 
  \end{enumerate}

  \item For any theoretical claim, check if you include:
  \begin{enumerate}
    \item Statements of the full set of assumptions of all theoretical results. [Yes]
    \item Complete proofs of all theoretical results. [Yes, Section~\ref{appendix:missing_proofs} in the supplementary material]
    \item Clear explanations of any assumptions. [Yes]
  \end{enumerate}

  \item For all figures and tables that present empirical results, check if you include:
  \begin{enumerate}
    \item The code, data, and instructions needed to reproduce the main experimental results (either in the supplemental material or as a URL). [\url{https://github.com/neerajgangwar/giva}]
    \item All the training details (e.g., data splits, hyperparameters, how they were chosen). [Yes, Section~\ref{appendix:training_details} in the supplementary material]
    \item A clear definition of the specific measure or statistics and error bars (e.g., with respect to the random seed after running experiments multiple times). [Yes]
    \item A description of the computing infrastructure used. (e.g., type of GPUs, internal cluster, or cloud provider). [Yes, Section~\ref{appendix:training_details} in the supplementary material]
  \end{enumerate}

  \item If you are using existing assets (e.g., code, data, models) or curating/releasing new assets, check if you include:
  \begin{enumerate}
    \item Citations of the creator if your work uses existing assets. [Yes]
    \item The license information of the assets, if applicable. [Not Applicable, we have used publicly available models and datasets]
    \item New assets either in the supplemental material or as a URL, if applicable. [Not Applicable]
    \item Information about consent from data providers/curators. [Not Applicable, we have used publicly available models and datasets]
    \item Discussion of sensible content if applicable, e.g., personally identifiable information or offensive content. [Not Applicable]
  \end{enumerate}

  \item If you used crowdsourcing or conducted research with human subjects, check if you include:
  \begin{enumerate}
    \item The full text of instructions given to participants and screenshots. [Not Applicable]
    \item Descriptions of potential participant risks, with links to Institutional Review Board (IRB) approvals if applicable. [Not Applicable]
    \item The estimated hourly wage paid to participants and the total amount spent on participant compensation. [Not Applicable]
  \end{enumerate}

\end{enumerate}

\appendix
\onecolumn
\thispagestyle{empty}
\aistatstitle{\vega{}: \papertitle{}\\Supplementary Materials}

\setcounter{theorem}{0}

\section{MISSING PROOFS}
\label{appendix:missing_proofs}
\begin{theorem}
Consider the following optimization problem
\begin{equation*}
    \label{appeq:optim_rewrite}
    \argmin_{A, B} \Vert \nabla_{W} \mathcal{L} (W_\pt) A^T B^T B A - \nabla_W \mathcal{L} (W_\pt) \Vert_F
\end{equation*}
Here, $\nabla_{W} \mathcal{L} (W_\pt) \in \mathbb{R}^{m \times d}$,  $B \in \mathbb{R}^{m \times r}$, $A \in \mathbb{R}^{r \times d}$, and $r < \operatorname{min}(m, d)$. The optimization problem reaches its minimum when
\begin{equation*}
    A = V_{r}^T, B^T B = \mathbb{I}^{r \times r}
\end{equation*}
$V_r$ denotes the right singular vectors of $\nabla_W \mathcal{L} (W_\pt)$, corresponding to the $r$ largest singular values.
\end{theorem}

\begin{proof}
The rank of the first term in the optimization problem can be computed as
\begin{equation*}
    \operatorname{rank}\left( \nabla_{W} \mathcal{L} (W_\pt) A^T B^T B A \right) \leq \operatorname{min} \left( \operatorname{rank} \left( \nabla_{W} \mathcal{L} (W_\pt) \right), \operatorname{rank}\left( A^T B^T B A \right) \right) \leq r   
\end{equation*}
Using the Eckart-Young theorem, the optimization problem reaches its minimum when $\nabla_{W} \mathcal{L} (W_\pt) A^T B^T B A$ is the best rank-$r$ approximation of $\nabla_W \mathcal{L} (W_\pt)$ \citep{eckart1936approximation}.

We prove the theorem by substituting the optimal $A$ and $B$ in the first term of the optimization problem. Let $U_r$ and $V_r$ denote the left and right singular vectors of $\nabla_W \mathcal{L} (W_\pt)$ corresponding to the $r$ largest singular values $\left\lbrace \sigma_1 \dots \sigma_r \right\rbrace$. Define $\Sigma_r = \operatorname{diag}(\sigma_1 \dots \sigma_r)$, and let $V_{d - r}$ denote the remaining right singular vectors.
\begin{equation*}
\begin{split}
    \nabla_{W} \mathcal{L} (W_\pt) A^T B^T B A &= U \Sigma V^T V_r V_r^T \\
    &= U \Sigma \begin{bmatrix}
        V_r^T \\
        V_{d - r}^T
    \end{bmatrix} V_r V_r^T \\
    &= U \Sigma \begin{bmatrix}
    \mathbb{I}^{r \times r} \\ \mathbf{0}^{d - r \times r}
    \end{bmatrix} V_r^T \\
    &= U_r \Sigma_r V_r^T
\end{split}
\end{equation*}
which is the best rank-$r$ approximation of $\nabla_W \mathcal{L} (W_\pt)$.
\end{proof}

\section{TRAINING DETAILS}
\label{appendix:training_details}
Our experiments use the \texttt{transformers} \citep{wolf2019huggingface} and \texttt{peft} \citep{peft} libraries from HuggingFace. For GLUE, we use a single A10 GPU (24GB VRAM), and the remaining experiments are conducted on a single A100 GPU (40GB VRAM) or a single L40S GPU (48GB VRAM). Gradient accumulation is applied when a full batch does not fit into memory.

\subsection{GLUE}
\label{sec:glue_training_details}
Table~\ref{tab:glue_hparams} shows the common hyperparameters used for fine-tuning RoBERTa on various GLUE tasks. Additionally, we search for the best learning rate over $\left\lbrace10^{-5}, 5 \times 10^{-5}, 10^{-4}\right\rbrace$ for full fine-tuning, over $\left\lbrace 5 \times 10^{-5}, 10^{-4}, 2 \times 10^{-4} \right\rbrace$ for LoRA and RandLoRA, and over $\left\lbrace 10^{-3}, 5 \times 10^{-3}, 10^{-2} \right\rbrace$ for VeRA and \vega{}. We do not use dropout or bias for any adapter modules. Finally, we use $\alpha = 2r$ for LoRA and RandLoRA and $d_\text{initial} = 0.1$ for VeRA.

\begin{table}[p]
    \caption{Common hyperparameters for fine-tuning RoBERTa on the GLUE tasks.}
    \label{tab:glue_hparams}
    \begin{center}
    \begin{small}
    \begin{tabular}{cccc}
    \toprule
    \textbf{Optimizer} & \textbf{Weight Decay} & \textbf{Warmup Steps} & \textbf{Epochs} \\
    \midrule
    AdamW & $10^{-2}$ & 0.06 & 10 \\
    \midrule
    \textbf{Gradient Clipping} & \textbf{Input Length} & \textbf{Batch Size} & \textbf{LR Scheduler} \\
    \midrule
    Norm-based at 1.0 & 128 & 64 & Linear \\
    \bottomrule
    \end{tabular}
    \end{small}
    \end{center}
\end{table}

\begin{table}[p]
    \caption{Common hyperparameters for fine-tuning \qwen{}, \msphi{}, and \olmo{} on commonsense reasoning.}
    \label{tab:csr_hparams}
    \begin{center}
    \begin{small}
    \begin{tabular}{cccc}
    \toprule
    \textbf{Optimizer} & \textbf{Weight Decay} & \textbf{Warmup Steps} & \textbf{Epochs} \\
    \midrule
    AdamW & $0$ & 100 & 3 \\
    \midrule
    \textbf{Gradient Clipping} & \textbf{Max Input Tokens} & \textbf{Batch Size} & \textbf{LR Scheduler} \\
    \midrule
    Norm-based at 1.0 & 512 & 16 & Linear \\
    \bottomrule
    \end{tabular}
    \end{small}
    \end{center}
\end{table}

\begin{table}[p]
    \caption{Common hyperparameters for fine-tuning \olmo{} on MetaMathQA and Code-Feedback.}
    \label{tab:math_code_hparams}
    \begin{center}
    \begin{small}
    \begin{tabular}{cccc}
    \toprule
    \textbf{Optimizer} & \textbf{Weight Decay} & \textbf{Warmup Steps} & \textbf{Epochs} \\
    \midrule
    AdamW & $0$ & 0.03 & 1 \\
    \midrule
    \textbf{Gradient Clipping} & \textbf{Max Input Tokens} & \textbf{Batch Size} & \textbf{LR Scheduler} \\
    \midrule
    Norm-based at 1.0 & 1024 & 32 & Cosine \\
    \bottomrule
    \end{tabular}
    \end{small}
    \end{center}
\end{table}

\begin{table}[p]
    \caption{Common hyperparameters for fine-tuning \mistral{} on Alpaca.}
    \label{tab:alpaca_hparams}
    \begin{center}
    \begin{small}
    \begin{tabular}{ccccc}
    \toprule
    \textbf{Optimizer} & \textbf{Weight Decay} & \textbf{Warmup Steps} & \textbf{Epochs} \\
    \midrule
    AdamW & $10^{-4}$ & 0.1 & 1 \\
    \midrule
    \textbf{Gradient Clipping} & \textbf{Max Input Tokens} & \textbf{Batch Size} & \textbf{LR Scheduler} \\
    \midrule
    Norm-based at 1.0 & 1024 & 16 & Cosine \\
    \bottomrule
    \end{tabular}
    \end{small}
    \end{center}
\end{table}

\begin{table}[p]
    \caption{Common hyperparameters for fine-tuning the vision backbones of \dinob{} and \clipl{} on CIFAR100, Food101, Flowers102, and RESISC45.}
    \label{tab:vision_hparams}
    \begin{center}
    \begin{small}
    \begin{tabular}{cccccc}
    \toprule
    \textbf{Optimizer} & \textbf{Weight Decay} & \textbf{Warmup Steps} & \textbf{Gradient Clipping} & \textbf{Batch Size} & \textbf{LR Scheduler} \\
    \midrule
    AdamW & $10^{-1}$ & 0.0 & Norm-based at 1.0 & 128 & Cosine \\
    \bottomrule
    \end{tabular}
    \end{small}
    \end{center}
\end{table}

\subsection{Commonsense Reasoning}
\label{sec:csr_training_details}
Table~\ref{tab:csr_hparams} shows the common hyperparameters used for fine-tuning models on commonsense reasoning. Additionally, we search for the best learning rate over $\left\lbrace 10^{-5}, 5 \times 10^{-5}, 10^{-4}, 2 \times 10^{-4} \right\rbrace$ for LoRA and RandLoRA, over $\left\lbrace 5 \times 10^{-4}, 10^{-3}, 5 \times 10^{-3}, 10^{-2} \right\rbrace$ for OSoRA, and over $\left\lbrace 10^{-3}, 5 \times 10^{-3}, 10^{-2} \right\rbrace$ for VeRA and \vega{}. We use a dropout of 0.05 and no bias for adapter modules. Finally, we use $\alpha = 2r$ for LoRA, $\alpha = \frac{2r}{\sqrt{n}}$ for RandLoRA with $n$ being the number of bases, and $d_\text{initial} = 1.0$ for VeRA.

\subsection{Mathematical Reasoning and Code Generation}
\label{sec:math_reasoning_training_details}
Table~\ref{tab:math_code_hparams} shows the common hyperparameters used for fine-tuning \olmo{} on MetaMathQA and Code-Feedback. Additionally, we search for the best learning rate over $\lbrace 5 \times 10^{-5}, 10^{-4}, 2 \times 10^{-4}, 5 \times 10^{-4}\rbrace$ for LoRA and over $\lbrace10^{-3}, 5 \times 10^{-3}, 10^{-2}, 2 \times 10^{-2}, 5 \times 10^{-2}\rbrace$ for VeRA, OSoRA, and \vega{}. We do not use dropout or bias for any adapter modules. Finally, we use $\alpha = 2r$ for LoRA and $d_\text{initial} = 1.0$ for VeRA.

\subsection{Instruction Tuning}
\label{sec:instruction_tuning_training_details}
Table~\ref{tab:alpaca_hparams} shows the common hyperparameters used for fine-tuning. Additionally, we search for the best learning rate over $\lbrace 5 \times 10^{-6}, 10^{-5}, 5 \times 10^{-5}, 10^{-4}, 4 \times 10^{-4}\rbrace$ for LoRA and over $\lbrace 10^{-4}, 5 \times 10^{-4}, 10^{-3}, 4 \times 10^{-3}, 5 \times 10^{-3}, 10^{-2}\rbrace$ for VeRA and \vega{}. We do not use dropout or bias for any adapter module. Finally, we use $\alpha = 2r$ for LoRA and $d_\text{initial} = 1.0$ for VeRA.

\subsection{Image Classification}
\label{sec:img_class_training_details}
Table~\ref{tab:vision_hparams} shows the common hyperparameters used for image classification experiments. Additionally, we search for the best learning rate over $\left\lbrace 10^{-5}, 2 \times 10^{-5}, 5 \times 10^{-5} \right\rbrace$ for full fine-tuning, over $\left\lbrace 5 \times 10^{-4}, 10^{-3}, 2 \times 10^{-3} \right\rbrace$ for LoRA and RandLoRA, and over $\left\lbrace 5 \times 10^{-3}, 10^{-2}, 2 \times 10^{-2} \right\rbrace$ for VeRA and \vega{}. We use 10 epochs for CIFAR100, 15 epochs for Food101 and RESISC45, and 40 epochs for Flowers102. No dropout or bias is used for any adapter module. Finally, we use $\alpha = 1$ for both LoRA and RandLoRA and $d_\text{initial} = 1.0$ for VeRA.

\end{document}